\documentclass{article}

\usepackage[preprint, final]{tagds_2025}

 \workshoptitle{Topology, Algebra, and Geometry in Data Science}

\usepackage[utf8]{inputenc} 
\usepackage[T1]{fontenc}    
\usepackage{hyperref}       
\usepackage{url}            
\usepackage{booktabs}       
\usepackage{amsfonts}       
\usepackage{nicefrac}       
\usepackage{microtype}      
\usepackage{xcolor}         
\usepackage[table]{xcolor}

\usepackage{enumitem}

\usepackage{algorithm}
\usepackage{algpseudocode}
\usepackage{graphicx}

\usepackage{amsmath}
\usepackage{adjustbox}
\usepackage{bbm}
\usepackage{amsthm}
\newtheorem{definition}{Definition}
\newtheorem{lemma}{Lemma}

\newtheorem{claim}{Claim}

\def\cast{{
   \mathord{
      \hbox to 0em{
         \ooalign{
	   \smash{\hbox{$\ast$}}\crcr
	   \smash{\hskip-1pt\Large\hbox{$\circ$}} }
	 \hidewidth}
      \phantom{\bigcirc}
} }}

\def\bm#1{\mbox{\boldmath $#1$}}

\newcommand{\vpsi}{\mbox{$\bm \psi$}}

\newcommand{\bds}{\begin {itemize}}
\newcommand{\eds}{\end {itemize}}
\newcommand{\bdf}{\begin{definition}}
\newcommand{\blm}{\begin{lemma}}
\newcommand{\edf}{\end{definition}}
\newcommand{\elm}{\end{lemma}}
\newcommand{\bthm}{\begin{theorem}}
\newcommand{\ethm}{\end{theorem}}
\newcommand{\bprp}{\begin{prop}}
\newcommand{\eprp}{\end{prop}}
\newcommand{\bcl}{\begin{claim}}
\newcommand{\ecl}{\end{claim}}
\newcommand{\bcr}{\begin{coro}}
\newcommand{\ecr}{\end{coro}}
\newcommand{\bquest}{\begin{question}}
\newcommand{\equest}{\end{question}}

\newcommand{\larrow}{{\larrow}}

\newcommand{\argmin}{\ensuremath{\mathrm{arg}\min}}
\newcommand{\argmax}{\ensuremath{\mathrm{arg}\max}}

\newcommand{\cG}{{\ensuremath{\mathcal{G}}}}

\newcommand{\vu}{{\ensuremath{{\mathbf{u}}}}}

\newcommand{\mA}{{\ensuremath{\mathbf{A}}}}

\newcommand{\mD}{{\ensuremath{\mathbf{D}}}}

\newcommand{\mL}{{\ensuremath{\mathbf{L}}}}

\usepackage{latexsym}

\def\IC{\mathbb C}
\def\IN{\mathbb N}
\def\IZ{\mathbb Z}
\def\IR{\mathbb R}

\def\shat{^{\mathchoice{}{}%
 {\,\,\smash{\hbox{\lower4pt\hbox{$\widehat{\null}$}}}}%
 {\,\smash{\hbox{\lower3pt\hbox{$\hat{\null}$}}}}}}

\def\bSigma{{
      \ooalign{
      \smash{\hskip.4pt\raise.4pt\hbox{$\Sigma$}}\vphantom{}\crcr
      \smash{\hskip.7pt\raise.6pt\hbox{$\Sigma$}}\vphantom{}\crcr
      \smash{\hbox{$\Sigma$}}\vphantom{$\Sigma$}}
      \vphantom{\hbox{$\Sigma$}}
      }}
\def\bTheta{{
      \ooalign{
      \smash{\hskip.5pt\raise.5pt\hbox{$\Theta$}}\vphantom{}\crcr
      \smash{\hskip.0pt\raise.1pt\hbox{$\Theta$}}\vphantom{}\crcr
      \smash{\hbox{$\Theta$}}\vphantom{$\Theta$}}
      \vphantom{\hbox{$\Theta$}}
      }}
\def\bDelta{{
      \ooalign{
      \smash{\hskip.4pt\raise.4pt\hbox{$\Delta$}}\vphantom{}\crcr
      \smash{\hskip.7pt\raise.6pt\hbox{$\Delta$}}\vphantom{}\crcr
      \smash{\hbox{$\Delta$}}\vphantom{$\Delta$}}
      \vphantom{\hbox{$\Delta$}}
      }}
\def\bLambda{{
      \ooalign{
      \smash{\hskip.5pt\raise.5pt\hbox{$\Lambda$}}\vphantom{}\crcr
      \smash{\hskip.0pt\raise.1pt\hbox{$\Lambda$}}\vphantom{}\crcr
      \smash{\hbox{$\Lambda$}}\vphantom{$\Lambda$}}
      \vphantom{\hbox{$\Lambda$}}
      }}

\makeatletter

\def\bordermatrix#1{\begingroup \m@th
  \@tempdima 8.75\p@
  \setbox\z@\vbox{%
    \def\cr{\crcr\noalign{\kern2\p@\global\let\cr\endline}}%
    \ialign{$##$\hfil\kern2\p@\kern\@tempdima&\thinspace\hfil$##$\hfil
      &&\quad\hfil$##$\hfil\crcr
      \omit\strut\hfil\crcr\noalign{\kern-\baselineskip}%
      #1\crcr\omit\strut\cr}}%
  \setbox\tw@\vbox{\unvcopy\z@\global\setbox\@ne\lastbox}%
  \setbox\tw@\hbox{\unhbox\@ne\unskip\global\setbox\@ne\lastbox}%
  \setbox\tw@\hbox{$\kern\wd\@ne\kern-\@tempdima\left[\kern-\wd\@ne
    \global\setbox\@ne\vbox{\box\@ne\kern2\p@}%
    \vcenter{\kern-\ht\@ne\unvbox\z@\kern-\baselineskip}\,\right]$}%
  \null\;\vbox{\kern\ht\@ne\box\tw@}\endgroup}
\makeatother

\makeatletter
\def\argmin{\mathop{\operator@font arg\,min}}
\def\argmax{\mathop{\operator@font arg\,max}}
\makeatother

\newcommand{\Tr}{\mbox{\rm Tr}}

\def\bm#1{\mbox{\boldmath $#1$}}

\newcommand{\bea}{\begin{array}}
\newcommand{\ena}{\end{array}}
\newcommand{\beq}{\begin{equation}}
\newcommand{\enq}{\end{equation}}

\newcommand{\beqa}{\begin{eqnarray}}
\newcommand{\enqa}{\end{eqnarray}}

\newcommand{\beqan}{\begin{eqnarray*}}
\newcommand{\enqan}{\end{eqnarray*}}

\newcommand{\AL}{\begin{enumerate}}
\newcommand{\ALE}{\end{enumerate}}

\def\addots{\mathinner{
    \mkern1mu\raise0pt\vbox{\kern7pt\hbox{.}}
    \mkern2mu\raise4pt\hbox{.}
    \mkern2mu\raise7pt\hbox{.}
    \mkern1mu}}

\def\sddots{\mathinner{
    \mkern.8mu\raise7pt\hbox{.}
    \mkern.8mu\raise4pt\hbox{.}
    \mkern.8mu\raise0pt\vbox{\kern7pt\hbox{.}}
    \mkern1mu}}

\def\saddots{\mathinner{
    \mkern.2mu\raise0pt\vbox{\kern7pt\hbox{.}}
    \mkern.2mu\raise4pt\hbox{.}
    \mkern.2mu\raise7pt\hbox{.}
    \mkern1mu}}

\def\sqplus{\mathbin{
	{\ooalign{\hfil\raise.3ex\hbox{\scriptsize
	+}\hfil\crcr\mathhexbox274\crcr\mathhexbox275}}
	}} 
\def\sqminus{\mathbin{
	{\ooalign{\hfil\raise.3ex\hbox{\scriptsize
	--}\hfil\crcr\mathhexbox274\crcr\mathhexbox275}}
	}}

\def\IC{{
   \mathord{
      \hbox to 0em{
	 \hskip-4pt
         \ooalign{
	   \smash{\hskip1.9pt\raise2.6pt\hbox{$\scriptscriptstyle |$}}\crcr
	   \smash{\hbox{\rm\sf C}} }
	 \hidewidth}
      \phantom{\hbox{\rm\sf C}}
} }}
\def\IN{
    {\ooalign{
   \smash{\hskip2.2pt\raise1.5pt\hbox{$\scriptscriptstyle |$}}\vphantom{}\crcr
   \hbox{\sf N}
	}}
	} 
\def\IZ{
    {\ooalign{
   \smash{\hskip1.9pt\raise0pt\hbox{$\sf Z$}}\vphantom{}\crcr
   \hbox{\sf Z}
	}}
	} 
\def\IR{
    {\ooalign{
   \smash{\hskip2.2pt\raise1.5pt\hbox{$\scriptscriptstyle |$}}\vphantom{}\crcr
   \smash{\hskip2.2pt\raise3.3pt\hbox{$\scriptscriptstyle |$}}\vphantom{}\crcr
   \hbox{\sf R}
	}}
	} 

\DeclareMathAlphabet{\mathcmb}{OT1}{cmr}{b}{n}

\def\bSigma{\ensuremath{\mathcmb{\Sigma}}}
\def\bLambda{\ensuremath{\mathcmb{\Lambda}}}

\def\bTheta{\ensuremath{\mathcmb{\Theta}}}

\newcommand{\SI}{\begin{indlist}}
\newcommand{\EI}{\end{indlist}}

\newcommand{\DL}{\begin{dashlist}}
\newcommand{\DLE}{\end{dashlist}}

\makeatletter
\def\setboxz@h{\setbox\z@\hbox}
\def\wdz@{\wd\z@}
\def\boxz@{\box\z@}
\def\underset#1#2{\binrel@{#2}%
  \binrel@@{\mathop{\kern\z@#2}\limits_{#1}}}
\def\binrel@#1{\begingroup
  \setboxz@h{\thinmuskip0mu
    \medmuskip\m@ne mu\thickmuskip\@ne mu
    \setbox\tw@\hbox{$#1\m@th$}\kern-\wd\tw@
    ${}#1{}\m@th$}%
  \edef\@tempa{\endgroup\let\noexpand\binrel@@
    \ifdim\wdz@<\z@ \mathbin
    \else\ifdim\wdz@>\z@ \mathrel
    \else \relax\fi\fi}%
  \@tempa
}
\let\binrel@@\relax%
\makeatother

\title{A Graph Laplacian Eigenvector-based Pre-training Method for Graph Neural Networks}

\author{
\textbf{Howard Dai}$^{1,*}$ \quad
\textbf{Nyambura Njenga}$^{2,*}$ \quad
\textbf{Hiren Madhu}$^{1}$ \quad
\textbf{Siddharth Viswanath}$^{1}$ \quad \\
\textbf{Ryan Pellico}$^{3,\dagger}$ \quad
\textbf{Ian Adelstein}$^{1,\dagger}$ \quad
\textbf{Smita Krishnaswamy}$^{1,\dagger}$ \\
\\
$^1$Yale University \quad
$^2$Georgia Institute of Technology \quad
$^3$Trinity College \\
$^*$co-first authors \quad
$^\dagger$co-senior authors \\
\textit{Correspondence}: \texttt{smita.krishnaswamy@yale.edu}
}

\begin{document}

\maketitle

\begin{abstract}
The development of self-supervised graph pre-training methods is a crucial ingredient in recent efforts to design robust graph foundation models (GFMs). \textit{Structure-based} pre-training methods are under-explored yet crucial for downstream applications which rely on underlying graph structure. In addition, pre-training traditional message passing GNNs to capture global and regional structure is often challenging due to the risk of oversmoothing as network depth increases. We address these gaps by proposing the Laplacian Eigenvector Learning Module (LELM), a novel pre-training module for graph neural networks (GNNs) based on predicting the low-frequency eigenvectors of the graph Laplacian. Moreover, LELM introduces a novel architecture that overcomes oversmoothing, allowing the GNN model to learn long-range interdependencies. Empirically, we show that models pre-trained via our framework outperform baseline models on downstream molecular property prediction tasks.
\end{abstract}

\section{Introduction}
Graph Neural Networks (GNNs) have become a powerful tool in node and graph representation learning, with successful applications across domains ranging from biomedicine \citep{canturk2023graph, yan2024empowering} to social networks \citep{fan2019graph}. More recently, graph foundation models (GFMs) are emerging as an exciting field; inspired by the success of large language models (LLMs), researchers are exploring the possibility of creating large graph-based models with emergent capabilities across many domains \citep{liu2025graph, xie2022self, wang2025graph}. 

A key ingredient in this effort is the creation of self-supervised tasks which can be performed on large unlabeled graph datasets \citep{liu2022graph}. A variety of pre-training methods have been proposed, with the majority of such methods being feature-based, including contrastive losses and feature reconstruction \citep{liu2025graph, xie2022self}. A few structure-based methods, which precompute labels based on graph topology, have been proposed \citep{peng2020self, hwang2020self}. However, this category of pre-training approach, recently termed \textit{graph property prediction} \citep{liu2025graph}, remains under-explored. Furthermore, limitations of GNNs in capturing long-range structural information caused by oversmoothing and underreaching \citep{cai2020note, oono2019graph, keriven2022not} pose challenges when approaching graph property prediction tasks which rely on global structure.


We propose the Laplacian Eigenvector Learning Module (LELM) for GNNs and GFMs. The low-frequency Laplacian eigenvectors capture a range of global, regional, and local graph structure, making them well-suited as a structural target. LELM utilizes a novel global MLP prediction head during pre-training that allows the GNN model to learn long-range relationships without requiring excessively deep networks, and it augments pre-training data with positional features to overcome expressivity limits of GNNs. LELM is highly flexible: it can be used to pre-train any GNN to improve downstream performance across all kinds of graph-based datasets, and it can be used either as an independent pre-training method or as a plug-in addition to existing graph pre-training pipelines. 


Our main contributions are as follows:

\begin{enumerate}[leftmargin=*,left=0pt..1em,topsep=0pt,itemsep=0pt]
    \item We propose LELM, a Laplacian eigenvector-based pre-training module for GNNs and GFMs. 
    \item We introduce a novel MLP head that enables GNNs to capture large-scale structure. Furthermore, we propose a set of augmented node features based on the graph diffusion operator.
    \item We empirically demonstrate that LELM provides performance improvements over baseline models. 
\end{enumerate}

\section{Background} \label{notation}
Given a graph $\cG = (V, E)$ with unnormalized adjacency matrix $\mA$ and degree matrix $\mD$, the unnormalized Laplacian $\mL$ of a graph $\cG$ is defined as $\mL = \mD - \mA$.

Let $\lambda_1, \lambda_2, \dots, \lambda_k$ denote the $k$ lowest eigenvalues of $\mL$ in nondecreasing order. Let $\psi_1, \psi_2, \dots \psi_k$ denote the corresponding eigenvectors, such that we have:
$$\mL \vpsi_i = \lambda_i \vpsi_i$$

The low-frequency graph Laplacian eigenvectors have been used for positional encodings, spectral GNNs, spectral clustering, and to generate provably minimal graph cuts. We provide examples and references for each of these applications in Section \ref{sec: applications}.

\section{Method}
\begin{figure}[t]
	\centering
	\includegraphics[width=1\linewidth]{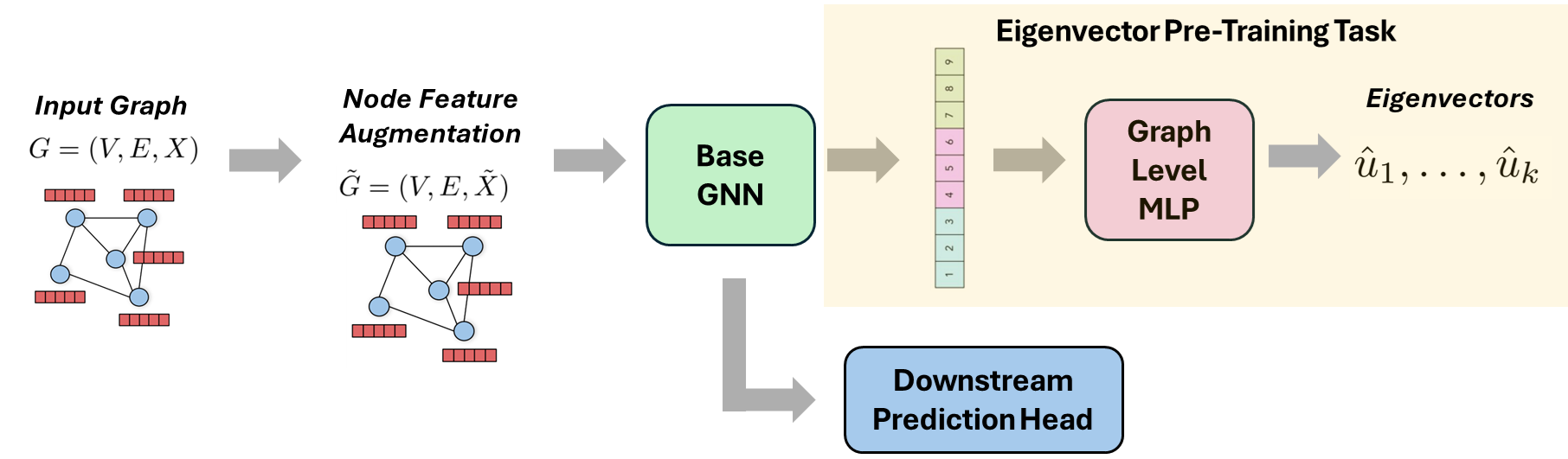}
	\caption{Overview of the LELM pre-training pipeline. Here, ``Base GNN" and ``Downstream Prediction Head" can be any user-defined model architecture. }
	\label{fig:sample}
\end{figure}

The LELM pre-training framework consists of three primary components:
\begin{itemize}[leftmargin=*,left=0pt..1em,topsep=0pt,itemsep=0pt]
    \item \textbf{Node feature augmentation:} We provide initial features based on the graph diffusion operator. These embeddings provide the base GNN model with additional structural context which is useful for the pre-training task as well as for downstream structure-aware tasks.
    \item \textbf{Graph-level MLP:} We pass a graph-level aggregated representation from a base GNN into our prediction MLP head.  The graph-level MLP overcomes the challenge of oversmoothing by allowing the base GNN to learn long-range node interactions without requiring an excessively deep message passing network.
    \item \textbf{Eigenvector prediction:} During pre-training, we task the model to predict the $k$ lowest-frequency eigenvectors of the graph Laplacian.  
\end{itemize}
 Together, these components allow the base GNN to learn local and global graph structure while overcoming intrinsic expressivity limitations, making LELM a robust pre-training framework for structure-aware downstream tasks. A visual overview of our pipeline can be seen in Figure \ref{fig:sample}, and a detailed algorithmic overview of our pipeline can be found in Section \ref{sec: Algorithms}.

\subsection{Node feature augmentation}
We propose two kinds of embeddings: \textbf{(1) wavelet positional embeddings}, which encode relative positional information between nodes, and \textbf{(2) diffused dirac embeddings}, which encode local connectivity structures around each node. Both embeddings use the graph diffusion operator, and capture local aggregate information on each node.  We provide more details as well as theoretical guarantees in Section \ref{sec: NodeFeatures}.

\subsection{Graph-level MLP} \label{sec: ModelArchitecture}
Once we perform node augmentation, we pass the augmented graph as input into a base GNN model and graph-level MLP to output predicted eigenvectors.\\\\
\textbf{Base GNN:} The base GNN model takes in a graph with augmented node features and generates learned node representations via neighborhood message passing and update steps.  Any GNN architecture may be selected as the base model to fit the needs of the dataset and downstream application.

\textbf{Graph-level MLP:} We concatenate the node-wise output of the base GNN model to form a graph-level aggregated representation.  We then pass the aggregated vector through an MLP model to produce the low-frequency Laplacian eigenvectors.  

Eigenvector-learning methods for other applications (such as spectral clustering) use a node-wise MLP head, processing each node's eigencoordinates independently based on their learned hidden embedding \citep{shaham2018spectralnet, dwivedi2021graph, canturk2023graph}.  Such methods fail to address oversmoothing as a node-wise MLP cannot learn long-range interactions; instead, one must use several layers of message passing within the base GNN \citep{canturk2023graph}.

\subsection{Eigenvector prediction} \label{sec: LossFunctions}

During pre-training, the model output aims to minimize a weighted sum of two loss functions: \textbf{(1) eigenvector loss} and \textbf{(2) energy loss}. 

To ensure the model does not output $k$ copies of the trivial eigenvector, we impose orthogonality on the final outputs of the model via QR decomposition, as proposed by \citet{shaham2018spectralnet}. Let each $\hat{\vu}_i$ denote the $i$th predicted eigenvector via LELM after orthogonality has been imposed. \\

\textbf{Energy loss}, used in \cite{shaham2018spectralnet, dwivedi2021graph, ma2023self}, computes the sum of Rayleigh quotients:
\[\mathcal{L}_{energy} = \frac{1}{k} \sum_{i=1}^k \hat{\vu}_i^\top \mL\hat{\vu}_i  \]
Energy loss is minimized when the predicted eigenvectors span the same subspace as the ground-truth eigenvectors; to enforce a strict ordering and basis on the predicted eigenvectors, we additionally impose \textbf{eigenvector loss}, used in \cite{mishne2019diffusion}: 
\[\mathcal{L}_{eigvec} = \frac{1}{k} \sum_{i=1}^k \lVert \mL\hat{\vu}_i - \lambda_i \hat{\vu}_i \rVert \]

We discuss further intuition behind our loss functions in Section \ref{sec: LossFunctionDetails} and provide proofs for necessary sign and basis invariances in Section \ref{sec: lossfn proofs}.

\renewcommand{\algorithmicrequire}{\textbf{Input:}}
\renewcommand{\algorithmicensure}{\textbf{Output:}}

\section{Experimental Results} \label{sec: standalone}
To evaluate the effectiveness of LELM, we conduct various experiments on real-world datasets.


\textbf{Comparison against baseline models:} We pre-train a standard Graph Isomorphism Network (GIN) \citep{GIN} and GPS \citep{rampavsek2022recipe}, a graph transformer, using LELM.  Once the model has been pre-trained, we replace the graph-level MLP head with a downstream prediction MLP and fine-tune model weights.  We evaluate our pre-training framework on three molecular datasets ZINC, ZINC-12k \citep{sterling2015zinc} and QM9 \citep{ramakrishnan2014quantum}. For each of these models, we compare LELM against the same GNN model without pre-training.  For both models, pre-training improves performance for all but one of the downstream targets. We record results of our experiments in Table \ref{tab: zinc}.
\begin{table} 
\caption{Test MAE ($\downarrow$) performance comparison on ZINC (single metric) and QM9 (first seven target properties).}
\label{tab: zinc}
\centering
\begin{adjustbox}{max width=\linewidth}
\begin{tabular}{l c c c ccccccc}
\toprule
 &  \multicolumn{1}{c}{ZINC full} & \multicolumn{1}{c}{ZINC subset} & \multicolumn{7}{c}{QM9} \\
 \cmidrule(lr){2-2} \cmidrule(lr){3-3} \cmidrule(lr){4-10}
Model   & Penalized $\log p$ & Penalized $\log p$ & $\mu$ & $\alpha$ & $\varepsilon_{\text{HOMO}}$ & $\varepsilon_{\text{LUMO}}$ & $\Delta_\varepsilon$  & $\{R^2\}$ & ZPVE\\
\midrule

\rowcolor{lightgray!50} GIN + LELM  &0.130 & 0.353 & 0.484 & \textbf{0.489} & \textbf{0.00353} & \textbf{0.00371} & 0.00513 & \textbf{28.103} & \textbf{0.00048}\\

GIN (baseline) & 0.228 & 0.438 & 0.472 & 1.132 & 0.00386 & 0.00399 & 0.00562 & 50.909 & 0.00240\\
\rowcolor{lightgray!50} GPS + LELM  & \textbf{0.104} & \textbf{0.210} & 0.502 & 0.592 & 0.00372 & 0.00408 & \textbf{0.00511} & 33.606 & 0.00178\\

GPS (baseline) & 0.150 & 0.358 & \textbf{0.413} & 0.718 & 0.00434 & 0.00442 & 0.00592 & 80.503 & 0.00111 \\
\bottomrule
\end{tabular}
\end{adjustbox}
\end{table}

\textbf{Comparison against alternative pre-training targets:} We compare LELM to alternative structure-based pre-training targets, including node degree, local clustering coefficient, cycle counting, and Laplacian eigenvalues. For each alternative target, we pre-train a standard GIN (with our default settings) using L1 loss. Results are recorded in Table \ref{tab:alt}. Here, we see that LELM outperforms alternative structure-based pre-training targets, indicating that the Laplacian eigenvectors are a better pre-training target than many other natural choices for graph targets.
\begin{table}[t] 
\caption{Test MAE ($\downarrow$) performance comparison when pre-training a GIN on ZINC to predict alternative structural training targets.}
\label{tab:alt}
\centering
\begin{adjustbox}{max width=\linewidth}
\begin{tabular}{l c c}
\toprule
Alternative training targets & ZINC full & ZINC subset \\
\midrule
\rowcolor{lightgray!50}LELM & \textbf{0.130} &\textbf{ 0.353} \\
Node degree & 0.238 & 0.471 \\
Local clustering coefficient & 1.493 & 1.551 \\
Cycle counting & 0.285 & 0.420 \\
Lap Eigenvalues & 0.250 & 0.520 \\
\bottomrule
\end{tabular}
\end{adjustbox}
\end{table}

\textbf{Ablation on MLP head:} We pre-train both the GIN and GPS using LELM, replacing the graph-level MLP with a standard MLP which acts on each node's embeddings individually. Results can be seen in Table \ref{tab:ablation}; for both the GIN and GPS, pre-training is more effective with the graph-level MLP.

\begin{table}[H]
\caption{Test MAE ($\downarrow$) performance comparison on ZINC when replacing the graph-level MLP with a basic node-wise MLP during pre-training.}
\label{tab:ablation}
\centering
\begin{adjustbox}{max width=\linewidth}
\begin{tabular}{l c c}
\toprule
Model & ZINC full & ZINC subset \\
\midrule
\rowcolor{lightgray!50}GIN + LELM & 0.130 & 0.353 \\
GIN + LELM (no graph-level MLP) & 0.152 &0.435 \\
\rowcolor{lightgray!50}GPS + LELM & \textbf{0.104} & \textbf{0.210} \\
GPS + LELM (no graph-level MLP) & 0.126 & 0.261 \\

\bottomrule
\end{tabular}
\end{adjustbox}
\end{table}

\textbf{Enhancing existing pre-training methods:} To demonstrate the flexibility of our method, we complement an existing pre-training method with LELM and demonstrate performance improvements in the majority of downstream tasks. Full experimental details and results can be found in Section \ref{sec: pretrain_augment}.

\section{Limitations and Future Work} \label{ref: limitations}

There are several promising future directions toward improving the LELM pre-training framework. First, we have demonstrated the effectiveness of the framework for pre-training and fine-tuning a GNN on the same dataset, or on domain-related datasets, but we have yet to explore the effectiveness of LELM as a \textit{transfer learning} framework. A future direction could be exploring the viability of LELM when pre-training on synthetic graphs or on cross-domain datasets. 
\newpage
\bibliography{tagds_2025}
\bibliographystyle{tagds_2025}


\appendix

\section{Related works: GNN pre-training methods}
Towards the goal of improving graph foundation models, a variety of self-supervised graph pre-training tasks have been proposed. According to the taxonomy provided by \citet{liu2025graph, xie2022self}, existing graph pre-training methods can be categorized into two broad categories: \textit{contrastive} and \textit{predictive} methods. 

Contrastive methods maximize mutual information between pairs of data views using objectives like Jensen-Shannon estimator \citep{nowozin2016f} or InfoNCE \citep{oord2018representation}. Methods can be categorized by the types of views used: graph-node \citep{suninfograph, velivckovicdeep, peng2020graph}, subgraph-node \citep{hustrategies, jiao2020sub}, and subgraph-subgraph \cite{qiu2020gcc}. Some methods also employ graph augmentation to generate two views \citep{you2020graph}. 

Predictive methods, also referred to as generative methods \citep{liu2025graph}, self-generate labels and train to predict these labels. A first class of predictive models uses graph reconstruction, whether by using node/edge masking \citep{xie2020noise2same, batson2019noise2self, hustrategies} or using autoencoders \citep{wang2017mgae, kipf2016variational}. A second class of predictive methods are \textit{property prediction} methods, which precompute underlying graph properties as labels. Examples include statistical properties such as k-hop connectivity \citep{peng2020self} or topological properties like a meta-path \citep{hwang2020self}. Overall, there are a lack of works on property prediction-based methods, with the majority of predictive pre-training methods falling under the former category of graph reconstruction \citep{liu2025graph}. Our method, LELM, is the first property prediction method to use the graph Laplacian eigenvectors as a pre-training target. 

\section{Graph Laplacian eigenvector applications} \label{sec: applications}
\textbf{Provably minimal graph cuts:} The second-lowest eigenvector $\psi_2$, known as the Fiedler vector, can be used to generate a provably ``good" cut on a graph; in particular, for some arbitrary threshold $s \in \mathbb{R}$, we can define a Fiedler cut $C$ to be:
$$C = (\{i: \psi_2(i) < s\}, \{i: \psi_2(i) \geq s\}) $$
On any bounded-degree $n$-vertex planar graph, the optimal Fiedler cut has ratio $O(\frac{1}{n}) $ \citep{spielman1996spectral}.

\textbf{Positional encodings:} The low-frequency Laplacian eigenvectors naturally encode a global position on the graph. As a result, Laplacian positional encodings (LapPE) have been used as a standard positional encoding for graph transformer models \citep{dwivedi2020generalization, rampavsek2022recipe}. In practice, directly using the Laplacian eigenvectors as positional encodings creates sign and basis ambiguity issues, as $\psi_i$ is an eigenvector of $L \iff -\psi_i$ is an eigenvector of $L$. Approaches to solving this problem include designing an architecture which processes the Laplacian eigenvectors in a sign- and basis-invariant manner \citep{lim2022sign} or defining canonical directions for the eigenvectors \citep{ma2023laplacian}. 

\textbf{Spectral GNNs:} A variety of methods, known as spectral graph neural networks, use the Laplacian eigendecomposition to learn filters in signal domain \citep{bo2023survey}. Some methods explicitly compute or approximate the $k$ lowest-frequency Laplacian eigenvectors, learning advanced filters on the corresponding eigenvalues \citep{bruna2013spectral, liao2019lanczosnet, bo2023specformer}. Other methods instead learn polynomial filters on the graph \citep{defferrard2016convolutional, he2022convolutional}, circumventing the expensive process of eigendecomposition by learning a $k$-degree polynomial function with respect to $L$, i.e. $p(L) = \theta_0 I  + \theta_1 L + \dots + \theta_k L^k$.

\textbf{Spectral clustering:} The Laplacian eigenvectors have also been used for clustering applications. Given a set of data $x_1, \dots x_n$,  \citet{belkin2001laplacian} construct a weighted graph $G$ with $n$ nodes using a heat kernel. Then to generate a $k$-dimensional embedding, \citet{belkin2001laplacian} compute the $k$-lowest eigenvectors $\psi_2, \dots \psi_{k+1}$ (omitting the trivial eigenvector) of the graph Laplacian and assign data point $x_i$ the embedding $(\psi_2(i), \psi_3(i), \dots , \psi_{k+1}(i))$. \citet{shaham2018spectralnet, chen2022specnet2} learn this spectral map using a neural network, allowing for a natural extension of this map to new datapoints.

\section{Complete algorithms} \label{sec: Algorithms}
We provide a detailed overview of the eigenvector pre-training process in Algorithm \ref{alg: EigenvectorPrediction}, and an overview of the full pre-training and fine-tuning pipeline in Algorithm \ref{alg:cap}. 

\begin{algorithm}[H]
\caption{Eigenvector Prediction} \label{alg: EigenvectorPrediction}
\begin{algorithmic}[1] 
\Require Graph $G = (V, E)$; augmented node features $\Tilde{X} = \{\Tilde{x}_j\}$; Base GNN
\Ensure Output Pre-trained GNN model, $k$ lowest-frequency eigenvectors
\For {$i < \text{Pre-Training Epochs}$}
\State $\Vec{z}_0, \dots, \Vec{z}_n \gets \textsc{BaseGNN}(G, \Tilde{X}) $
\State $\vec{Z} \gets [\Vec{z}_1, \dots, \Vec{z}_n] \in \mathbb{R}^{nd}$
\State $\Tilde{U} \gets \textsc{MLP}(\Vec{Z})$
\State $\hat{U} = \textsc{QR}(\Tilde{U})$
\State $\textit{Loss} = \alpha\cdot\textsc{EnergyLoss}(\hat{U}) + \beta \cdot \textsc{EigvecLoss}(\hat{U})$
\State Back-propagate Loss, update model weights
\EndFor \\
\Return \textsc{BaseGNN}
\end{algorithmic}
\end{algorithm}

\begin{algorithm}[H]
\caption{Structure-Informed Graph Pre-training Framework}\label{alg:cap}
\begin{algorithmic}[1]
\Require Graph $G = (V, E)$; node features $X = \{x_j\}$; training labels $Y$; untrained Base GNN; untrained Downstream Prediction Head
\Ensure Trained Base GNN and Downstream Prediction Head
\State $\Tilde{X} \gets \textsc{AugmentFeatures}(G, X)$
\State \textsc{BaseGNN} $\gets \textsc{EigvecPreTrain}(G, \Tilde{X}, \textsc{BaseGNN})$
\For {$i < \text{Fine-tuning Epochs}$}
\State $\Vec{z}_0, \dots, \Vec{z}_n \gets \textsc{BaseGNN}(G, \Tilde{X}) $
\State $\vec{Z} \gets [\Vec{z}_1, \dots, \Vec{z}_n]$
\State $\hat{Y} \gets \textsc{DownstreamHead}(\Vec{Z})$
\State $\textit{Loss} = \textsc{LossCriterion}(\hat{Y}, Y)$
\State Backpropagate Loss, update model weights
\EndFor \\
\Return \textsc{BaseGNN, DownstreamHead}
\end{algorithmic}
\end{algorithm}

\section{Loss function} \label{sec: LossFunctionDetails}

We minimize a weighted sum of two loss functions: \textbf{(1) eigenvector loss} and \textbf{(2) energy loss}. Both loss functions respect necessary sign and basis invariances of Laplacian eigenvectors; full proofs can be found in \ref{sec: lossfn proofs}.

\textbf{Energy loss}, used by \citet{shaham2018spectralnet, dwivedi2021graph, ma2023self}, aims to minimize the sum of Rayleigh quotients:
\[ \mathcal{L}_{energy} = \frac{1}{k} \Tr (\hat{U}^\top L \hat{U}) \]

This loss function is motivated by the iterative optimization problem following from Courant-Fischer, which states that the eigenvectors of $L$ (and the eigenvectors of any Hermitian matrix) can be equivalently expressed as solutions to the following iterative optimization problem: 
\[
\psi_k \in 
\argmin_{\substack{\|x\|=1\\ x\perp \psi_1,\dots,\psi_{k-1}}}
x^\top L x .
\]
The term $\frac{x^\top L x}{x^\top x}$ is known as the Rayleigh quotient; because we normalize our predicted eigenvectors, we simply treat this as $x^\top L x$. 

However, minimizing this loss function only minimizes the \textit{sum} of the first $k$ Rayleigh quotients, meaning the minimizer (subject to orthogonality) is any set of vectors spanning same subspace spanned by the $k$ lowest frequency eigenvectors. For applications in clustering, this is reasonable, as the exact basis in which embeddings are expressed is often irrelevant; however, to require the model to truly predict the $k$-lowest eigenvectors, we must include a more explicit penalty, such as \textbf{eigenvector loss}. 

\textbf{Eigenvector loss}, used by \citet{mishne2019diffusion}, measures the difference between each $L\hat{u}_i$ and $\lambda_i \hat{u}_i$:
\[\mathcal{L}_{eigvec} = \frac{1}{k} \lVert (L \hat{U} - \hat{U} \Lambda_k ) \rVert \]

Eigenvector loss enforces both the correct basis and a strict ordering (up to eigenvalue multiplicity) on the predicted eigenvectors.
Our final loss function is then:
$$\mathcal{L} = \alpha \cdot \mathcal{L}_{energy} + \beta \cdot \mathcal{L}_{eigvec}$$

\section{Loss function alternative expressions}
Eigenvector loss, per-vector form: 

\[\mathcal{L}_{eigvec} = \frac{1}{k} \sum_{i=1}^k \lVert L\hat{u}_i - \lambda_i \hat{u}_i \rVert \]
Eigenvector loss, matrix form:
\[\mathcal{L}_{eigvec} = \frac{1}{k} \lVert (L \hat{U} - \hat{U} \Lambda_k ) \rVert \]

Energy loss, per-vector form:
\[\mathcal{L}_{energy} = \frac{1}{k} \sum_{i=1}^k \hat{u}_i^\top L\hat{u}_i  \]

Energy loss, matrix form:
\[ \mathcal{L}_{energy} = \frac{1}{k} \Tr (\hat{U}^\top L \hat{U}) \]

Energy loss is order-invariant and rotation invariant (see \ref{sec: lossfn proofs}); for applications in clustering, this is reasonable. However, we would like the model to learn the eigenvectors in their specific order, so we also define \textbf{absolute energy loss}, matching the Rayleigh quotient with the ground-truth eigenvalue:
\[\mathcal{L}_{energy\_abs} = \frac{1}{k} \sum_{i=1}^k \lvert \hat{u}_i^\top L\hat{u}_i - \lambda_i \rvert \]
This can be written as, in matrix form: 
\[\mathcal{L}_{energy\_abs} =  \frac{1}{k} \Tr \lvert \hat{U}^\top L \hat{U} - \Lambda_k \rvert \]

In practice, we do not show any results using absolute energy loss, and instead linearly combine energy loss with eigenvector loss to avoid order and rotation invariance. However, absolute energy loss remains an interesting avenue to explore. 

\subsection{Orthogonality}
To ensure the model does not output $k$ copies of the trivial eigenvector, we must give the model orthogonality constraints on the output vectors. There are again two reasonable choices here: \textbf{(1) forced orthogonality} and \textbf{(2) orthogonality loss}. 

\textbf{Forced orthogonality}, used in \cite{shaham2018spectralnet}, imposes orthogonality on the final outputs of the model via QR decomposition. In other words, if $\hat{U}'$ is the initial output to the model, $Q$ is an $n \times k$ matrix with orthonormal columns, and $R$ is a $k \times k$ upper triangular matrix, then we achieve the final output $\hat{U}$ as such:
\[QR = \hat{U}'\]
\[\hat{U} = Q\]

\textbf{Orthogonality loss}, used in \cite{dwivedi2021graph, ma2023self, mishne2019diffusion} imposes a softer constraint, encouraging orthogonality by penalizing the model for producing pairwise similar vectors. This can be written as:
\[\mathcal{L}_{ortho} = \frac{1}{k} \lVert \hat{U}^\top \hat{U} - I \rVert\]

Based on preliminary testing, we found that forced orthogonality improved performance on the eigenvector-learning, and thus use forced orthogonality in all of our experiments. 

\section{Energy and eigenvector losses are sign and basis invariant} \label{sec: lossfn proofs}

\subsection{Definition of basis invariance}
Consider any eigenspace spanned by ground truth eigenvectors $[\psi_j, \psi_{j+1}, \dots \psi_{j+k-1}] = V$. Also recall that, by Spectral Theorem, we can decompose any vector $u$ into a linear combination of all eigenvectors:
\[u = \sum_{i=1}^n c_i \psi_i \]
Then a loss function is basis invariant if any rotation of the projected component $VV^\top u$ does not change the loss incurred by $u$. In other words, $u$ gets to arbitrarily ``choose" with what basis it wants to express its $VV^\top u$ component. Sign invariance is a special case of basis invariance, where changing sign is equivalent to rotating over a one-dimensional subspace (note that this is slightly stronger than the most apparent form of sign invariance, where we would say $\mathcal{L}(u) = \mathcal{L}(-u)$; instead, we can flip any \textit{component} $c_i\psi_i$ of $u$ when decomposed in terms of eigenvectors).

\begin{definition}[Basis invariance]
    Consider an eigenspace spanned by ground truth eigenvectors $[\psi_j, \psi_{j+1}, \dots \psi_{j+k-1}] = \Psi \in \mathbb{R}^{n \times k}$. Consider an eigenspace rotation $R_\Psi$ defined as such:
\[R_\Psi = \Psi A \Psi^\top + (I_n - \Psi\Psi^\top), A \in \textsc{SO}(k)\]
A loss function $\mathcal{L}(u)$ is basis invariant if, for all such $\Psi, R_\Psi$, $u \in \mathbb{R}^n$, we have:
\[ L(u) = L(R_\Psi u)\]
\end{definition}

\subsection{Proofs}

\begin{lemma}[Energy loss is basis invariant]
For any $R_\Psi$ and a single eigenvector prediction $u \in \mathbb{R}^n$, we have:
\[u^\top L u = (R_\Psi u)^\top L (R_\Psi u)\]
\end{lemma}
\begin{proof}
    First note that $R_\Psi$ is orthogonal; the set of all $R_\Psi$ describes a subset of $SO(k)$ where only the $k$ basis vectors in $\Psi$ are rotated. Thus, we have $R_\Psi^\top R_\Psi = I$. 

    In addition, because $\Psi$ is an eigenspace, all columns are eigenvectors with a shared eigenvalue $\lambda$. Then we have:
    \[R_\Psi L = \Psi A \Psi^\top L + L - \Psi\Psi^\top L = 
\lambda \Psi A\Psi^\top  + L - \lambda \Psi \Psi^\top = L\Psi A \Psi^\top + L - L\Psi \Psi^\top = LR_\Psi  \]
Then we have:
\[R_\Psi^\top L R_\Psi = R_\Psi^\top R_\Psi L = L\]
Thus, for any $u$, we have:
\[u^\top L u = u^\top R_\Psi^\top L R_\Psi u = (R_\Psi u)^\top L (R_\Psi u)\]
\end{proof}

\begin{lemma}[Eigenvector loss is basis invariant]
    For any $R_\Psi$ and a single eigenvector prediction $u \in \mathbb{R}^n$ and ground truth eigenvalue $\lambda$, we have:
    \[  \lVert Lu - \lambda u \rVert  = \lVert L (R_\Psi u) - \lambda  (R_\Psi u) \rVert\]
\end{lemma}
\begin{proof}
    We know, from our proof above in Lemma 1, that $R_\Psi L = L R_\Psi$. 
    Because $R_\Psi \in \textsc{SO}(k)$, we have $\lVert R_\Psi x \rVert = \lVert x \rVert$ for any $x \in \mathbb{R}^n$. Then we have:
    \[  \lVert Lu - \lambda u \rVert  = \lVert R_\Psi (L u - \lambda  u) \rVert\]
    \[  \lVert Lu - \lambda u \rVert  = \lVert  L (R_\Psi  u) - \lambda  (R_\Psi u) \rVert\]
\end{proof}
We have an even stronger statement of invariance for energy loss: \textbf{given a predicted set of $k$ orthogonal vectors, rotating the vectors within the same subspace does not impact loss. }In other words, a model trained on energy loss only needs to predict the correct \textit{subspace} of $k$ eigenvectors. This is clearly not true of eigenvector loss. Depending on the application, this kind of invariance can be good or bad. 
\begin{lemma}[Energy loss is rotation invariant] Let $L$ be a Laplacian matrix and $V \subseteq \mathbb{R}^n$ be some $k$-dimensional subspace.  Suppose $U = [u_1, u_2, \dots, u_k], W = [w_1, w_2, \dots, w_k] \in \mathbb{R}^{n \times k}$ are both orthonormal bases for $V$. Then we have:
\[\frac{1}{k} \Tr (U^\top L U) = \frac{1}{k} \Tr (W^\top L W) \]
\end{lemma}

\begin{proof}
    Note that $UU^\top = WW^\top$, as they are both orthogonal projectors for the same subspace. Then we have, by the cyclic property of trace:
    $$\frac{1}{k} \Tr (U^\top L U) = \frac{1}{k} \Tr (UU^\top L) = \frac{1}{k} \Tr (WW^\top L) = \frac{1}{k} \Tr (W^\top L W)$$
\end{proof}


\section{Node feature augmentation} \label{sec: NodeFeatures}
\subsection{Graph Diffusion Operator}
The diffusion operator $P$ of a graph $G$ is defined as:
\[ P = D^{-1}A \]
Each entry $P_{ij}$ represents the probability of starting a random walk at node $i$ and ending at node $j$ after one step.  One can also take powers of the diffusion operator, $P^t$.  Each entry of the powered matrix,  $P^t_{ij}$, represents the probability of starting a random walk at node $i$ and ending at node $j$ after t steps

The $j^\text{th}$ wavelet operator $\Psi_j$ of a graph $G$ is defined as:
\[ \Psi_j = P^{2^{j-1}} - P^{2^{j}} \]
\[ \Psi_0 = I - P\]

A wavelet bank, $\mathcal{W}_J$ is a collection of wavelet operators such that:
\[ \mathcal{W}_J = \{\Psi_j\}_{0\leq j \leq J} \cup P^{2^J}\]

\subsection{Node Embeddings}
\textbf{Wavelet positional embeddings} encode information about the relative position of each node within the graph. We randomly select two nodes from each graph, $i$ and $j$, and start dirac signals $\delta_i, \delta_j$.  We then apply these signals to each wavelet, $\Psi_k$, in our wavelet bank.  The wavelet positional embedding for node $m$ is the $m^{\text{th}}$ row of the resulting matrix.

\[ w_{m,k} = \Psi_k(m, \cdot) \begin{bmatrix}
    | & |\\\delta_i & \delta_j\\| & |
\end{bmatrix} \]
\[ w_m = \begin{bmatrix}
    w_{m, 1} & \dots & w_{m,J}
\end{bmatrix} \]

\textbf{Diffused dirac embeddings} encode information about the connectedness of each node and its neighbors.  For each node, $m$, we apply the $m^\text{th}$ row of the diffusion matrix $P$ to each wavelet $\Psi_k$ in our wavelet bank. As above, the diffused dirac embedding for node $m$ is the $m^{\text{th}}$ row of the resulting matrix.
\[ d_{m, k} = \Psi_k(m, \cdot) ~P(m, \cdot)^\top \]
\[ d_m = \begin{bmatrix}
    d_{m, 1} & \dots & d_{m,J}
\end{bmatrix} \]

These node embeddings are unique up to co-spectrality of the graph Laplacian.

\begin{lemma}
    [Uniqueness up to co-spectrality] Let $G_1, G_2$ be graphs of size $n$ with Laplacian matrices $L_1, L_2$ respectively.  Let $d^1_m, d^2_m$ represent the diffused dirac embeddings for each node in $G_1, G_2$.  Then if $L_1$ and $L_2$ have different eigenvalues, $\{d^1_m: m \leq n\} \neq \{d_m^2: m\leq n\}$ 
\end{lemma}

\begin{proof}
Consider the random-walk Laplacian of a graph: $L_{rw} := I - D^{-1}A = I - P$.  Moreover, note that $L_{rw} = D^{-1}L$.  Observe that \begin{align*}
    L_{rw}Dv &= D^{-1}LDv\\
    &= D^{-1}U\Lambda U^\top Dv\\
    &= Bv \text{ for some diagonalizable matrix } B \text{ with eigenvalues } \lambda_i, \dots, \lambda_n\\
\end{align*}
Where $U = \begin{bmatrix}
    \psi_1 & \dots & \psi_n
\end{bmatrix}$, with $\psi_i$ orthonormal eigenvectors of $L$ and $\Lambda$ is the diagonal matrix of eigenvalues $\lambda_1, \dots, \lambda_n$ of $L$.
Any change to the eigenspectrum of $L$, clearly results in a change to $L_{rw}$, and therefore $P$.  Since $\Psi_0 = I - P$, any two graphs with distinct Laplacian eigenspectra will have distinct diffused dirac node embeddings.
\end{proof}

 \section{Alternative structural pre-training targets}
 Here, we formally define and provide details for the alternative pre-training targets used in section \ref{sec: standalone}.  
 \begin{itemize}
     \item \textbf{Node degree:} A node-level label representing the degree of each node
     \item \textbf{Local clustering coefficient:} A node-level label computing the local clustering coefficient of each node. For a fixed node $u$, the coefficient $C$ is given by: $${\displaystyle C={\frac {2|\{e_{jk}:v_{j},v_{k}\in N_{u},e_{jk}\in E\}|}{|N_{u}|(|N_{u}| - 1)}}.}$$
     \item \textbf{RWSE:} A node-level label computing self-walk probabilities at varying step counts for the diffusion operator \citep{dwivedi2021graph}. In our experiments, we use step counts from the interval $[2, 22]$. 
     \item \textbf{Cycle counting:} A graph-level label computing cycle counts of varying lengths. In our experiments, we count cycles up to length 9. 
     \item \textbf{Lap Eigenvalues:} A graph-level label computing the $k$-lowest Laplacian eigenvalues $\lambda_1, \dots, \lambda_k$. We use the same $k = 6$ as we do with LELM. 
 \end{itemize}

 For all alternative structural pre-training tasks, we use the same hyperparameters for GIN as displayed in \ref{tab: settings}, with no initial features and using a standard MAE loss instead of eigenvector + energy loss. We train on the full ZINC dataset. All structural pre-training targets are normalized to have mean 0 and standard deviation 1 across the entire dataset.


\section{Enhancing an existing Graph Neural Network pre-training method} \label{sec: pretrain_augment}
\subsection{Summary}
We augment the existing molecular pre-training methods proposed by \citet{hu2019strategies} with eigenvector-learning. In particular, \citet{hu2019strategies} propose node-level pre-training tasks (context prediction and masking) on ZINC15 \citep{sterling2015zinc}, followed by a graph-level supervised pre-training task on ChEMBL \citep{mayr2018large, gaulton2012chembl}. We augment the graph-level supervised pre-training step by adding an additional MLP head to predict eigenvectors, and we evaluate on five downstream datasets based on work by \citet{sun2022does}.

Detailed results are shown in Table \ref{tab: AugmentExisting}. Eigenvector-learning consistently improves performance for the masking pre-training pipeline, but achieves mixed results on the context prediction pipeline. Notably, performance for the masking pipeline was increased for all five datasets when performing eigenvector pre-training with the graph-level MLP head.

\begin{table}[t]\centering
\caption{Test ROC-AUC (\%, $\uparrow$) performance on 5 molecular prediction tasks when \textbf{augmenting an existing pre-training method} on a GIN base model. \textit{Sup.} refers to the original supervised pre-training as implemented by \citet{hu2019strategies}, and \textit{Sup.+} refers to supervised training with LELM. Results for \textit{no pre-training} are taken directly from \citet{sun2022does}. All methods are tuned over seven learning rates and averaged over three seeds.}  \label{tab: AugmentExisting}
\begin{adjustbox}{max width=\linewidth}
\begin{tabular}{lcccccc}\toprule
\multicolumn{2}{c}{Dataset} &BACE &BBBP &Tox21 &ToxCast &SIDER \\\cmidrule{1-7}
Pretrain method &MLP Head & & & & & \\\midrule
\rowcolor{lightgray!50} ContextPred, Sup.+ &Graph-level &$79.62 \pm 3.63$ &$\mathbf{70.76 \pm 1.64}$ &$77.94 \pm 0.11$ &$\mathbf{66.13 \pm 0.34}$ &$60.05 \pm 0.99$ \\
\rowcolor{lightgray!50} ContextPred, Sup.+ &Node-wise &$75.87 \pm 3.11$ &$68.74 \pm 1.07$ &$78.86 \pm 0.06$ &$63.78 \pm 0.32$ &$59.83 \pm 0.53$ \\

ContextPred, Sup. &- &$\mathbf{84.98 \pm 1.28}$ &$68.25 \pm 0.48$ &$77.44 \pm 0.19$ &$64.01 \pm 0.81$ &$\mathbf{62.87 \pm 0.89}$ \\
\rowcolor{lightgray!50} Masking, Sup.+ &Graph-level &$80.71 \pm 3.84$ &$68.33 \pm 0.89$ &$\mathit{79.09 \pm 0.25}$ &$\mathit{65.96 \pm 0.20}$ &$\mathit{62.41 \pm 1.77}$ \\

\rowcolor{lightgray!50} Masking, Sup.+ &Node-wise &$\mathit{81.02 \pm 1.67}$ &$\mathit{69.94 \pm 1.76}$ &$\mathbf{79.33 \pm 0.41}$ &$65.14 \pm 0.44$ &$59.38 \pm 1.11$ \\

Masking, Sup. &- &$75.42 \pm 2.64$ &$67.36 \pm 4.60$ &$78.33 \pm 0.24$ &$64.88 \pm 0.82$ &$61.6 \pm 1.78$ \\

No pre-training &- &$75.77 \pm 4.29$ &$69.62 \pm 1.05$ &$75.52 \pm 0.67$ &$63.67 \pm 0.32$ &$59.07 \pm 1.13$ \\
\bottomrule
\end{tabular}
\end{adjustbox}
\end{table}

\subsection{Learning rate tuning}
We keep the majority of the settings from \cite{hu2019strategies} the same. For downstream fine-tuning, we tune over 7 learning rates for fair comparison according to \cite{sun2022does}. We run each method and learning rate over 3 seeds, and select the learning rate based on mean validation accuracy over all learning rates. 

\subsection{Downstream datasets}
We briefly list and cite the five downstream datasets here for reference. The five datasets are the datasets chosen in \cite{sun2022does}, and are a subset of the eight primary downstream datasets evaluated in \cite{hu2019strategies}. 
\begin{itemize}
    \item \textbf{BACE:} Qualitative binding results \cite{subramanian2016computational}
    \item \textbf{BBBP:} Blood-brain barrier penetration \cite{martins2012bayesian}
    \item \textbf{Tox21:} Toxicity data \cite{mayr2016deeptox}
    \item \textbf{Toxcast:} Toxicology measurements \cite{richard2016toxcast}
    \item \textbf{SIDER:} Database of adverse drug reactions (ADR) \cite{kuhn2016sider}

\end{itemize}

\section{Detailed experimental settings} \label{sec: Detailed}
A complete overview of model hyperparameters and settings can be found in Table \ref{tab: settings}. Heuristically, the Graph-level MLP head hidden dimension is chosen to be the max \# nodes multiplied by the hidden dimension size of the base GNN. We do NOT omit the trivial eigenvector when counting number of eigenvectors predicted. 

\begin{table}[!htp]\centering
\caption{A comprehensive list of all model hyperparameters used during the eigenvector pre-training step. All hyperparameters highlighted in gray are specific to eigenvector-learning, while other listed configs reflect general GNN settings (and are set to match default values in each respective baseline work).} 
\begin{adjustbox}{max width=\linewidth}
    
\begin{tabular}{lccc}\toprule \label{tab: settings}
Method &GIN (\ref{sec: standalone}) &GPS (\ref{sec: standalone}) &GIN pre-training (\ref{sec: pretrain_augment}) \\\midrule
Pre-training dataset &ZINC-subset (12k), ZINC (250k), QM9 (134k) &ZINC-subset (12k), ZINC (250k), QM9 (134k) &ZINC15 (2M), ChEMBL (456K) \\
Base architecture &GIN &Transformer/GIN &GIN \\
\# params &33543 &157680 &2252210 \\
\# layers of per-node feature update &3 &3 &2 \\
\# layers of message passing &4 &4 &5 \\
Hidden dim &60 &60 &300 \\
Activation fn &ReLU &ReLU &ReLU \\
Dropout &0.1 &0.1 &0.2 \\
Batch size &128 &128 &32 \\
Learning rate &0.001 &0.001 &0.001 \\
Optimizer &Adam &Adam &Adam \\
Scheduler &ReduceLROnPlateau &ReduceLROnPlateau &None \\
 & patience=5, factor=0.9 &patience=20, factor=0.5 & - \\ 
Pre-Training Epochs &200 &100 &100 \\
Fine Tuning Epochs &500 &150 & 100 \\
\rowcolor{gray!50} \textbf{Laplacian norm type} &Unnormalized &Unnormalized &Unnormalized \\
 \rowcolor{gray!50} \textbf{\# eigenvectors predicted} &6 &6 &5 \\
\rowcolor{gray!50} \textbf{Initial features} &Diffused dirac + Wavelet pos. &Diffused dirac + Wavelet pos. &Molecule features \\
\rowcolor{gray!50} \textbf{MLP head type(s)} &Graph-level &Graph-level &Graph-level, per-node \\
\rowcolor{gray!50} \textbf{Graph-level MLP max \# nodes} &40 &40 &50 \\
\rowcolor{gray!50} \textbf{MLP head \# layers} &5 &5 &1 \\
\rowcolor{gray!50} \textbf{MLP head hidden dim} &2400 &2400 &N/A \\
\rowcolor{gray!50} \textbf{MLP head activation fn} &ReLU &ReLU &N/A \\
\rowcolor{gray!50} \textbf{Loss function (and coefficient)} &2*Eigenvector + 1*energy &2*Eigenvector + 1*energy &0.25 * Eigenvector + 0.05 * ortho \\
\rowcolor{gray!50} \textbf{Other features/notes} & & Removed graphs with less than six nodes during pre-training & \\
\bottomrule
\end{tabular}
\end{adjustbox}
\end{table}

For our ablation experiments with node-wise MLP heads, we keep all settings the same, replacing the graph-level MLP, which has an graph-level input dimension of $60 \cdot 40 = 2400$, with a node-level MLP which has an input dimension and hidden dimension of 60.

For ZINC-12k and ZINC, the test/train splits are provided.  
The QM9 dataset is randomly split into training,
validation, and test sets with a ratio of 0.8/0.1/0.1

\section{Runtime and compute} \label{sec:compute}
All experiments were run using Yale's Grace cluster. Runs used a single GPU each, with at least 11GB of RAM per GPU. The type of GPU varied between A100, V100, A5000, and RTX5000. For all runs using our main pipeline (comparison against baseline models, comparison against alternative pre-training targets, and ablation on MLP head), pre-training and fine-tuning on all downstream tasks took between 12 to 24 hours on ZINC-full, and 36 to 48 hours on QM9. For experiments enhancing an existing pre-training pipeline, the full pre-training process (including both unsupervised and supervised steps) took around 24 hours, and the downstream fine-tuning across all five datasets for a single seed and model variant took around 15 hours. 

Across all results that are included in this paper, we estimate a total of 325 GPU hours spent on the main-body experiments, and an additional 400 GPU hours spent on enhancing an existing GNN pre-training method. 

The complete experimental process consumed far more GPU hours, as we initially explored basic models' abilities to learn Laplacian eigenvectors, and also experimented with a variety of augmentation features before deciding on the two included in the paper. We estimate an additional 800 GPU hours used for initial experimentation and exploration.

\section{Visualization of predicted eigenvectors during pre-training}

\begin{figure}
    \centering
    \includegraphics[width=0.6\linewidth]{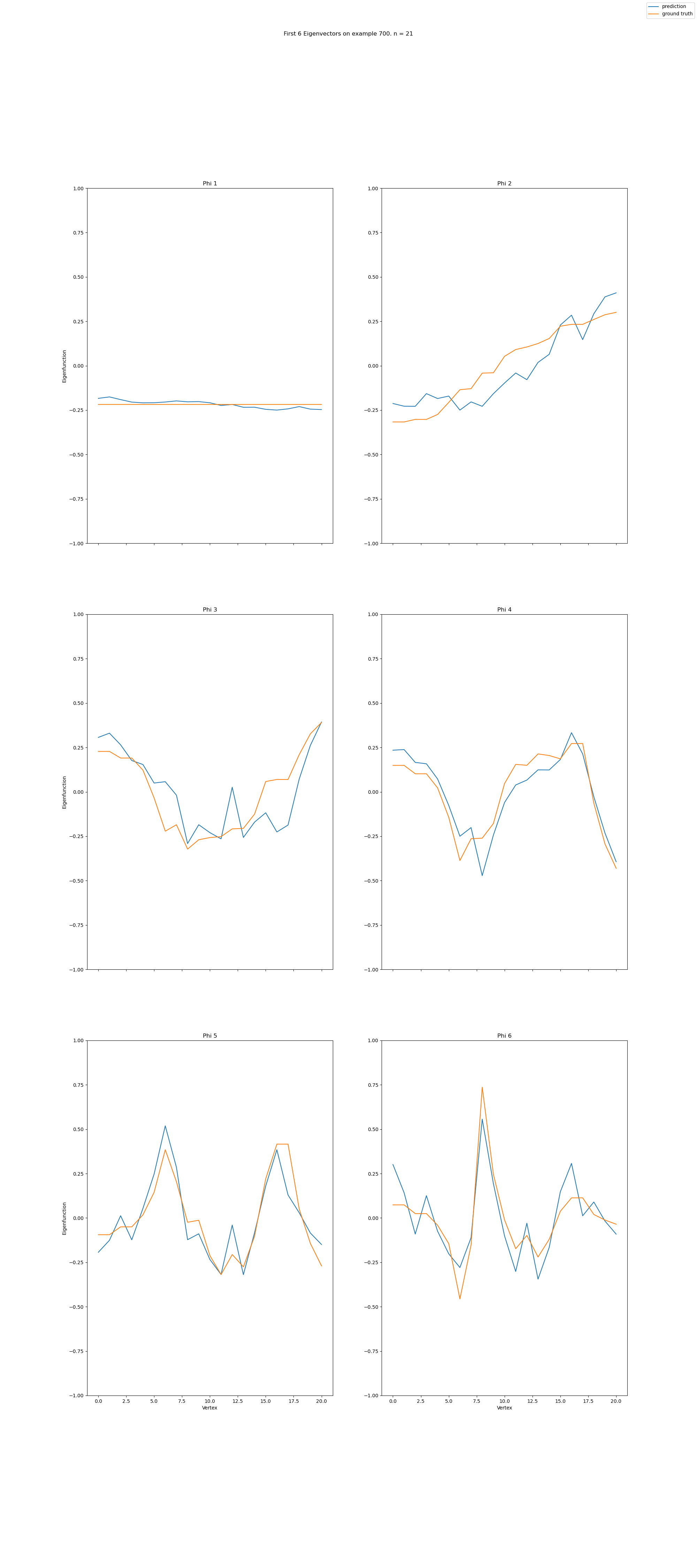}
    \caption{A comparison of predicted eigenvectors (blue) with ground-truth eigenvectors (orange) for a single molecular graph with $n = 21$ nodes. These predictions were produced by the standard GIN model on a validation example from the ZINC dataset. Node indices are sorted in increasing order of $\psi_2(i)$, and the sign orientation of the predicted vectors $\hat{u}_i$ is chosen such that $\hat{u}_i \cdot \psi_i\geq 0$. }
    \label{fig:placeholder}
\end{figure}

\newpage
\section*{TAG-DS Paper Checklist}

\begin{enumerate}

\item {\bf Claims}
    \item[] Question: Do the main claims made in the abstract and introduction accurately reflect the paper's contributions and scope?
    \item[] Answer: \answerYes{} 
    \item[] Justification: The claims of the abstract and introduction are that we developed a novel GNN pre-training method that improves performance on structure-based tasks.  The empirical results demonstrate the improvements, and the theoretical results justify the design of the framework.

\item {\bf Limitations}
    \item[] Question: Does the paper discuss the limitations of the work performed by the authors?
    \item[] Answer: \answerYes{} 
    \item[] Justification: The paper includes a limitations section, section  \ref{ref: limitations}, that discusses the question of whether the framework is suitable for transfer learning.

\item {\bf Theory assumptions and proofs}
    \item[] Question: For each theoretical result, does the paper provide the full set of assumptions and a complete (and correct) proof?
    \item[] Answer: \answerYes{} 
    \item[] Justification: All proofs in the paper include the full set of assumptions and all proofs are complete and correct.  Proofs are included in sections \ref{sec: lossfn proofs} and \ref{sec: NodeFeatures}

    \item {\bf Experimental result reproducibility}
    \item[] Question: Does the paper fully disclose all the information needed to reproduce the main experimental results of the paper to the extent that it affects the main claims and/or conclusions of the paper (regardless of whether the code and data are provided or not)?
    \item[] Answer: \answerYes{} 
    \item[] Justification: Code is provided in supplemental material and data is publicly accessible.  Moreover, all experimental settings are provided in supplemental material as well for ease of reproducibility.

\item {\bf Open access to data and code}
    \item[] Question: Does the paper provide open access to the data and code, with sufficient instructions to faithfully reproduce the main experimental results, as described in supplemental material?
    \item[] Answer: \answerYes{} 
    \item[] Justification: Code is provided as part of supplementary material and all data is publicly accessible.

\item {\bf Experimental setting/details}
    \item[] Question: Does the paper specify all the training and test details (e.g., data splits, hyperparameters, how they were chosen, type of optimizer, etc.) necessary to understand the results?
    \item[] Answer: \answerYes{} 
    \item[] Justification: All experimental details, including data splits, hyperparameters, optimizer, network depth, etc., are provided in the supplemental material in section \ref{sec: Detailed}.

\item {\bf Experiment statistical significance}
    \item[] Question: Does the paper report error bars suitably and correctly defined or other appropriate information about the statistical significance of the experiments?
    \item[] Answer: \answerNo{} 
    \item[] Justification: Error bars are not reported because it would be too computationally
expensive.

\item {\bf Experiments compute resources}
    \item[] Question: For each experiment, does the paper provide sufficient information on the computer resources (type of compute workers, memory, time of execution) needed to reproduce the experiments?
    \item[] Answer: \answerYes{} 
    \item[] Justification: We provide a complete overview of compute resources used and time of execution in the appendix in section \ref{sec:compute}.
    
\item {\bf Code of ethics}
    \item[] Question: Does the research conducted in the paper conform, in every respect, with the NeurIPS Code of Ethics \url{https://neurips.cc/public/EthicsGuidelines}?
    \item[] Answer: \answerYes{} 
    \item[] Justification: Paper does not violate any part of Neurips Code of Ethics.

\item {\bf Broader impacts}
    \item[] Question: Does the paper discuss both potential positive societal impacts and negative societal impacts of the work performed?
    \item[] Answer: \answerNA{} 
    \item[] Justification: The framework developed in this paper is not tied to any particular application.
    
\item {\bf Safeguards}
    \item[] Question: Does the paper describe safeguards that have been put in place for responsible release of data or models that have a high risk for misuse (e.g., pretrained language models, image generators, or scraped datasets)?
    \item[] Answer: \answerNA{} 
    \item[] Justification: All models are trained publicly accessible molecular property prediction datasets that pose no safety risks.

\item {\bf Licenses for existing assets}
    \item[] Question: Are the creators or original owners of assets (e.g., code, data, models), used in the paper, properly credited and are the license and terms of use explicitly mentioned and properly respected?
    \item[] Answer: \answerNA{} 
    \item[] Justification:The paper does not use existing assets. 

\item {\bf New assets}
    \item[] Question: Are new assets introduced in the paper well documented and is the documentation provided alongside the assets?
    \item[] Answer: \answerNA{} 
    \item[] Justification: The paper does not release new assets.

\item {\bf Crowdsourcing and research with human subjects}
    \item[] Question: For crowdsourcing experiments and research with human subjects, does the paper include the full text of instructions given to participants and screenshots, if applicable, as well as details about compensation (if any)? 
    \item[] Answer: \answerNA{} 
    \item[] Justification: The paper does not involve crowdsourcing nor research with human subjects.

\item {\bf Institutional review board (IRB) approvals or equivalent for research with human subjects}
    \item[] Question: Does the paper describe potential risks incurred by study participants, whether such risks were disclosed to the subjects, and whether Institutional Review Board (IRB) approvals (or an equivalent approval/review based on the requirements of your country or institution) were obtained?
    \item[] Answer: \answerNA{} 
    \item[] Justification: The paper does not involve crowdsourcing nor research with human subjects.

\item {\bf Declaration of LLM usage}
    \item[] Question: Does the paper describe the usage of LLMs if it is an important, original, or non-standard component of the core methods in this research? Note that if the LLM is used only for writing, editing, or formatting purposes and does not impact the core methodology, scientific rigorousness, or originality of the research, declaration is not required.
    \item[] Answer: \answerNA{} 
    \item[] Justification: LLMs were not used for core methods in this research.  The only use of LLMs was for checking grammar in the paper.

\end{enumerate}

\pagebreak
\end{document}